\documentclass[letterpaper]{article} 
\usepackage{aaai24}  
\usepackage{times}  
\usepackage{helvet}  
\usepackage{courier}  
\usepackage[hyphens]{url}  
\usepackage{graphicx} 
\usepackage{natbib}  
\usepackage{caption} 
\usepackage{algorithm}
\usepackage{algorithmic}
\usepackage{amsfonts}
\usepackage{amsmath}

\usepackage{amsthm}
\usepackage{amssymb}
\usepackage{mathtools}
\usepackage{graphics}
\newtheorem{theorem}{Theorem}[section]
\newtheorem{assumption}{Assumption}[section]
\newtheorem{lemma}[theorem]{Lemma}

\usepackage{color}
\usepackage{courier}
\usepackage{multirow}
\usepackage{subcaption,booktabs}
\usepackage{soul}
\usepackage{bm}

\usepackage{newfloat}
\usepackage{listings}
\DeclareCaptionStyle{ruled}{labelfont=normalfont,labelsep=colon,strut=off} 
\floatstyle{ruled}
\newfloat{listing}{tb}{lst}{}
\floatname{listing}{Listing}
\title{Generalizable Task Representation Learning for Offline Meta-Reinforcement Learning with Data Limitations}
\author{
    Renzhe Zhou,  
    Chen-Xiao Gao,
    Zongzhang Zhang\thanks{Zongzhang Zhang is the corresponding author.},
    Yang Yu
}
\affiliations{
    National Key Laboratory for Novel Software Technology, Nanjing University, China \\
    School of Artificial Intelligence, Nanjing University, China\\

    \{zhourz, gaocx\}@lamda.nju.edu.cn,
    \{zzzhang, yuy\}@nju.edu.cn
}

\usepackage{bibentry}

\begin{document}

\maketitle

\begin{abstract}
Generalization and sample efficiency have been long-standing issues concerning reinforcement learning, and thus the field of Offline Meta-Reinforcement Learning~(OMRL) has gained increasing attention due to its potential of solving a wide range of problems with static and limited offline data. Existing OMRL methods often assume sufficient training tasks and data coverage to apply contrastive learning to extract task representations. However, such assumptions are not applicable in several real-world applications and thus undermine the generalization ability of the representations. In this paper, we consider OMRL with two types of data limitations: limited training tasks and limited behavior diversity and propose a novel algorithm called GENTLE for learning generalizable task representations in the face of data limitations. GENTLE employs Task Auto-Encoder~(TAE), which is an encoder-decoder architecture to extract the characteristics of the tasks. Unlike existing methods, TAE is optimized solely by reconstruction of the state transition and reward, which captures the generative structure of the task models and produces generalizable representations when training tasks are limited. To alleviate the effect of limited behavior diversity, we consistently construct pseudo-transitions to align the data distribution used to train TAE with the data distribution encountered during testing. Empirically, GENTLE significantly outperforms existing OMRL methods on both in-distribution tasks and out-of-distribution tasks across both the given-context protocol and the one-shot protocol. 
\end{abstract}

\section{Introduction}
Despite the success of Reinforcement Learning~(RL) in scenarios where online interaction is consistently available, RL is hampered from real-world applications such as healthcare and robotics controlling due to its sample complexity~\cite{sac} and inferior generalization ability~\cite{generalization_survey}. The past decade witnessed tremendous effort from researchers to pave the path for RL toward real-world applications. For example, offline RL~\cite{bcq,cql,td3bc,iql}, which optimizes the policies with a pre-collected and static dataset, provides a solution to relieving RL from costly online interactions, whereas meta-RL~\cite{rl2, maml, pearl, varibad, towards}, which involves training policies over a wide range of tasks, significantly enhances the generalization ability of the learned policies. 

Offline Meta-Reinforcement Learning~(OMRL)~\cite{mbml, focal, borel, macaw, corro}, as an intersection of offline RL and meta-RL, is promising to combine the good of both worlds. In OMRL, we are provided with datasets collected in various tasks which share some similarity in the underlying structures in dynamics or reward mechanisms, and aim to optimize the meta-policy. The meta-policy is later tested in tasks drawn from the same task distribution. Previous related methods~\cite{mbml, focal, corro} often interpret the OMRL challenge as task representation learning and meta-policy optimization. The former step aims to obtain indicative task representations from the dataset, while the latter optimizes a meta-policy on top of the learned representation. 
However, existing methods often assume a sufficient number of training tasks as well as sufficient diversity of behavior policy that collects the datasets, which is not realistic in real-world applications. We find that when the assumptions are not satisfied, the representations tend to overfit and fail to generalize on unseen testing tasks. 

In light of this, we propose a new approach to \textbf{Gen}eralizable \textbf{T}ask representations \textbf{Le}arning~(GENTLE) to enable effective task recognition in the face of limitations in training task quantity and behavior diversity. GENTLE follows the existing paradigm of OMRL and consists of two interleaving optimization stages: (1) task representation learning and (2) offline meta-policy optimization on top of the learned representations. For (1), we introduce a novel structure, Task Auto-Encoder~(TAE) to extract representations from the context information. TAE is optimized to reconstruct the state transition and rewards on the probing data rather than contrastive loss, which models the generative structure of the environment and prevents the encoder from overfitting to miscellaneous features when the number of training tasks is limited. To alleviate TAE's training from overfitting to the behavior policy distribution, we augment the training data via policy, dynamics, and reward relabeling, forcing TAE to learn to exploit the difference in dynamics and rewards rather than input data distributions. For (2), we adopt TD3+BC for its simplicity to optimize a meta-policy with task representations predicted by the TAE. 

For evaluations, we compare GENTLE against other baseline algorithms in a set of continuous control tasks with two types of evaluation protocols: \textit{given-context} protocol where the context is collected by an ad-hoc expert policy in the target environment, and \textit{one-shot} protocol where the context is collected by the meta-policy. Experimental results demonstrate the superiority of GENTLE over the baseline methods, and the ablation study also discloses the necessity of each component of GENTLE. 

\section{Related Work}
\subsubsection{Offline Meta-Reinforcement Learning. }
Generalization is a known issue about RL agents~\cite{generalization_survey}, and thus meta-RL is proposed to enhance the generalization ability of RL agents. Current meta-RL research can be categorized into two types: gradient-based approaches~\cite{maml, macaw, merpo}, which focus on fast adaptation to new tasks via few-shot gradient descent, and context-based approaches~\cite{rl2, pearl}, which formalize the meta-RL tasks as contextual Markov Decision Processes (MDPs) and learn to encode task representations from histories. The combination of meta-RL and offline setting leads to Offline Meta-Reinforcement Learning~(OMRL), a framework where only static task datasets are available to learn a meta-policy. Most of the previous OMRL methods follow the context-based approach~\cite{mbml,  focal, borel, smac, corro}. Overall, the workflow of these methods can break down into two procedures. The first is to learn a task representation encoder with the offline dataset and augment the states with the learned representations, while the second step is to optimize the meta-policy with offline RL algorithms. MACAW~\cite{macaw}, on the other hand, follows the gradient-based approach and extends MAML~\cite{maml} to the offline setting. Our method falls in the first category, with additional considerations for the data limitations. Similar to our motivations, MBML~\cite{mbml}, BOReL~\cite{borel}, and CORRO~\cite{corro} also identify the impact of behavior policy on task identification,  and alleviate this issue either by reward relabeling or generative relabeling on offline datasets. MIER~\cite{mier} and SMAC~\cite{smac}, on the other hand, focus on the meta-testing stage and mitigate this issue by relabeling the context collected online.

\subsubsection{Task Representation Learning. }
Successful meta-RL agent relies on the learned task representations to make adaptive decisions in different tasks. On how to encode the context and derive task representations, various methods differ in their learning objectives. Earlier methods, such as RL$^2$~\cite{rl2} and PEARL~\cite{pearl}, apply the same RL objective for the representation encoder. Specifically, RL$^2$ passes the gradient of the encoder through the representation and optimizes the encoder in an end-to-end fashion, while PEARL trains the encoder via critic loss combined with an additional information bottleneck term. Alternative approaches, like ESCP~\cite{escp}, employ the objective of maximizing relational matrix determinant of latent representations. In offline RL, a predominant approach to task representation learning is contrastive-style training~\cite{mbml, focal, corro}. These methods take advantage of the static datasets, construct positive pairs and negative pairs via relabeling or generative augmenting, and afterward apply contrastive objectives to optimize the encoder.
In this paper, we propose to extract representations via reconstruction, thus utilizing the generative structure of the underlying dynamics and facilitating the generalization of the representation. This idea has also been explored in past literature~\cite{varibad, sept, borel, mtsac, metadiffuser}. However, we apply this idea in the offline setting with off-policy data and characterize it theoretically. 

\section{Preliminaries}
\subsection{Problem Formulation}
The RL problem can be formulated as a Markov Decision Process (MDP), which can be characterized by a tuple $M=(\mathcal{S},\mathcal{A},T,R,\mu_0,\gamma) $, where $\mathcal{S}$ is the state space, $\mathcal{A}$ is the action space, $T(s'|s, a)$ is the transition function, $R(s, a)$ is the reward function, $\mu_0(s)$ is the initial state distribution, and $\gamma\in[0,1]$ is the discount factor. The policy $\pi(a|s)$ is a distribution over actions. The agent's goal is to find the optimal policy that maximizes the expected cumulative reward (a.k.a. return) $\mathop{\max}_\pi\eta(M,\pi)= \mathbb{E}_{\pi, M}[\sum_{t=0}^\infty \gamma^t R(s_t,a_t)]$, where the expectation is taken over the trajectory distribution which is induced by $\pi$ in $M$. The Q-function is defined as the expected return starting from state $s$, taking action $a$, and  thereafter
following policy $\pi$: $Q_\pi(s,a)=\mathbb{E}_{\pi, M}[\sum_{t'=t}^\infty \gamma^{t'-t} R(s_{t'},a_{t'})|s_t=s,a_t=a]$.

In Offline Meta-Reinforcement Learning (OMRL), we consider a set of tasks where each task is an MDP $M_i=(\mathcal{S},\mathcal{A},T_i,R_i,\mu_0,\gamma)$ and sampled from a task distribution $M_i\sim P(\mathcal{M})$. We assume the tasks only differ in transition functions and reward functions, and abbreviate them as $M=(T, R)$. We will use the term \textit{model} to refer to $M$ hereafter. During offline meta-training, we are given $N$ training tasks $\{M_i\}_{i=1}^N$ sampled from $P(\mathcal{M})$ and the corresponding offline datasets $\{D_i\}_{i=1}^N$ generated by behavior policies. Using the fixed offline datasets, the algorithm needs to train a meta-policy $\pi_\text{meta}$. During meta-testing, given a testing task $M\sim P(\mathcal{M})$, the agent first needs to identify the environment with context information $B^c$ before evaluation. Finally, the goal of meta-RL is to find the optimal meta-policy that maximizes the expected return over the task distribution:
\begin{equation}
    \mathop{\max}_{\pi_\text{meta}}\eta(\pi_\text{meta})= \mathbb{E}_{M\sim P(\mathcal{M})}[\eta(M,\pi_\text{meta})].
\end{equation}

\begin{figure*}[t]    
    \centering
\includegraphics[width=0.9\textwidth]{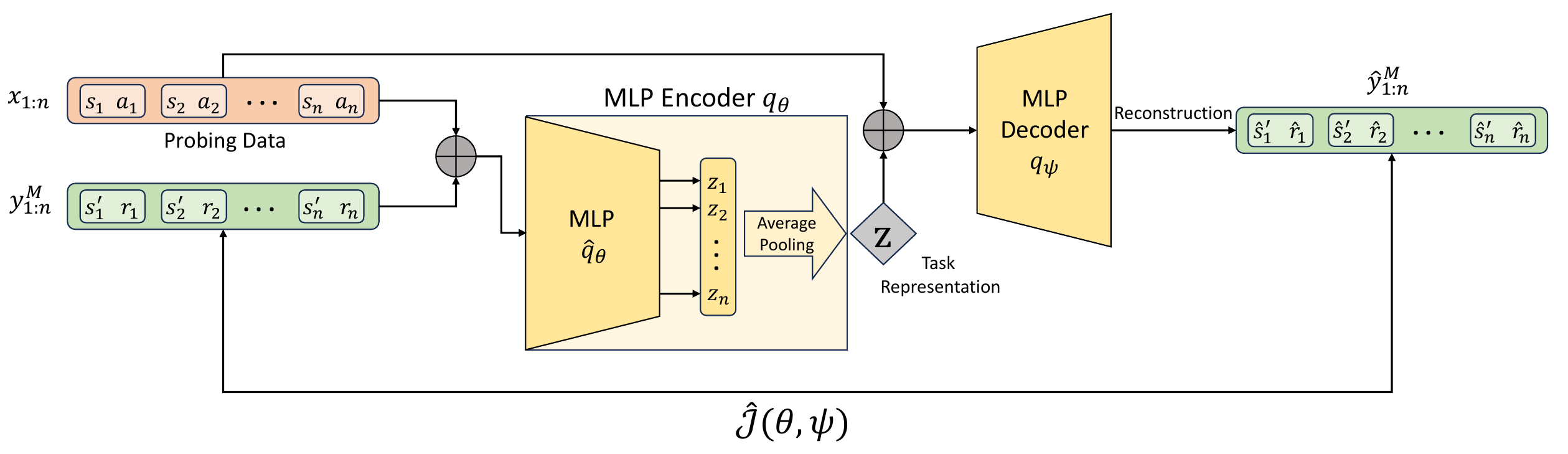}
    \caption{Illustration of TAE. The inputs of TAE are the concatenation of $x_{1:n}$ and $y_{1:n}^M$. The task representation $z$ is obtained by average pooling over the $n$ outputs of $\hat{q}_\theta$. And $q_\psi$ seeks to reconstruct $y_{1:n}^M$ given the probing data $x_{1:n}$ and representation $z$.}
    \label{sec4:fig_tae}
\end{figure*}

\subsection{TD3+BC}
Offline RL trains a policy solely on a fixed dataset $D$. Due to the mismatch between the distributions of the dataset and the current policy, the agent tends to falsely evaluate the value of out-of-distribution actions and thus misleads the policy optimization~\cite{bcq, bear}. To address the problem, TD3+BC~\cite{td3bc} adds a regularization term to the objective in TD3~\cite{td3} to constrain the policy around the dataset:
\begin{equation}
    \begin{aligned}
    \max_\pi\mathbb{E}_{(s,a)\sim D}[\lambda Q(s, \pi(s)) - (\pi(s)-a)^2],
    \end{aligned}
\end{equation}
where $D$ is the offline dataset and $\lambda$ is the coefficient balancing between the TD3's objective and the regularization term. We choose TD3+BC as our backbone offline RL algorithm due to its simplicity and effectiveness. And for generality, we use the stochastic form of policy to represent the meta-policy as $\pi(\cdot|s,z)$ hereafter.

\section{Method}

In this section, we will elaborate on the core ingredients of our GENTLE method, designed to learn generalizable task representations from offline datasets. 

\subsection{Task Auto-Encoder}\label{sec4.1}
We begin with our TAE, which applies an encoder-decoder architecture to learn task representations. First, we define $x\in \mathcal{S}\times\mathcal{A}$ as the probing data and $\rho(x)$ as its distribution. We additionally assume that the distribution of probing data is task-agnostic:

\begin{assumption}{(Independency of Probing Data)}\label{sec4:assump1}
The distribution of probing data is independent of the model, i.e., $p(x)=p(x|M)$ and $p(M|x)=p(M)$. 
\end{assumption}

For each entry of the probing data $x=(s, a)$, we first sample a model $M=(T_M, R_M)$ from the task distribution, and continue to sample a state transition $s'\sim T_M(\cdot|s, a)$ and the reward $r=R_M(s, a)$ to assign the label of $x$ as $y=(s', r)$. We use $x_{1:n}$ to represent $n$ i.i.d. samples from $\rho(x)$ and $y_{1:n}$ to represent the corresponding labels of $x_{1:n}$. Throughout this paper, we make a further assumption that the model of the environment is deterministic, i.e.:
\begin{assumption}{(Deterministic Model)}\label{sec4:assump2}
For input $x=(s, a)\in \mathcal{S}\times\mathcal{A}$ and model $M=(T_M, R_M)\in\mathcal{M}$, $p(y|x, M)=\delta(y=y^M)$. Here $y^M=M(x)$ and $M( x)\coloneqq (T_M(s, a), R_M(s, a))$. 
\end{assumption}

With all these prepared, we now describe TAE's architecture. In TAE, the encoder $q_\theta(z|x_{1:n}, y_{1:n})$ takes a batch of probing data $x_{1:n}$ and their labels $y_{1:n}$ as inputs, and outputs a distribution of the representation vector. The decoder $q_\psi(\hat{y}_k|x_k, z)$, on the other hand, takes the predicted representation and one of the probing data $x_k$ as input, and finally outputs the distribution of the predicted label. 

We train our TAE by maximizing the log-likelihood of the ground-truth label. Specifically, we maximize the following objective in terms of the parameters $\theta$ and $\psi$:
\begin{equation}\label{sec4:eq_tae}
    \mathcal{J}(\theta, \psi) = \mathbb{E}_{x_{1:n}, M, z}\Bigg[\sum_{k=1}^n\log q_\psi(y^M_k|z, x_k)\Bigg],
\end{equation}
where $x_{1:n}\sim \rho(x), {M}\sim P(\mathcal{M}), z\sim q_\theta(z|x_{1:n}, y^M_{1:n})$. Intuitively, the encoder takes in the labeled data and predicts the task representation $z$, while the decoder seeks to reconstruct the ground-truth label via the maximum log-likelihood over the batch of probing data. 

We now provide a theoretical characterization of the training process. 

\begin{theorem}\label{sec4:theorem}
Let $x_{1:n}$ denote i.i.d. probing data sampled from distribution $\rho(x)$ and  the probing data is labeled with sampled models from $P(\mathcal{M})$ to construct $y_{1:n}$. With Assumptions \ref{sec4:assump1} and \ref{sec4:assump2}, optimizing $\mathcal{J}(\theta, \psi)$ in terms of the representation encoder $q_{\theta}$ and the decoder $q_\psi$ corresponds to optimizing a lower bound of the mutual information between the task representation and the model $I(M; z)$: 
$$I(M; z)\geq I(M; y_{1:n}|x_{1:n}) + \mathcal{J}(\theta, \psi). $$
\end{theorem}
The proof is deferred to the appendix
\footnote{\url{https://www.lamda.nju.edu.cn/zhourz/AAAI24-supp.pdf}}. In Theorem~\ref{sec4:theorem}, we show that we can lower-bound the original mutual information $I(M; z)$ via two terms. The first term, $I(M; y_{1:n}|x_{1:n})$, measures how discriminative the probing data $x_{1:n}$ is in terms of models, while the second term is precisely the training objective of TAE which is a reconstruction-oriented objective. Thus, optimizing Equation~\eqref{sec4:eq_tae} corresponds to optimizing the lower bound of the mutual information between the extracted representation $z$ and the tasks, which justifies the effectiveness of our objective. It is noteworthy that optimizing mutual information is intractable in practice. A prior method, CORRO~\cite{corro}, derives a tractable lower bound via InfoNCE~\cite{cpc}, which employs a generative model to construct negative samples conditioned on the data within the task. Such a method focuses on all of the discriminative aspects of the input data, without consideration for the generative structure shared by different models. In contrast, our approach explicitly learns a decoder to account for this, thus enabling the encoder to closely approximate the intrinsic characteristics of the task. 

\subsection{Practical Implementation of TAE}\label{sec:4.2}
Section~\ref{sec4.1} presents a principal framework of TAE, while in this section, we will elaborate on the practical implementation of TAE, as illustrated in Figure~\ref{sec4:fig_tae}. 

The encoder $q_\theta$ is implemented as a deterministic mapping from a batch of probing data to a representation vector $z\in \mathbb{R}^m$. This module consists of a feature transformation network $\hat{q}_\theta$ and an average-pooling layer. For each pair of the input probing data $(x_k, y_k)$, $\hat{q}_\theta$ processes and projects the input into an intermediate embedding vector $z_k\in\mathbb{R}^m$, and the average-pooling layer aggregates the embeddings of each pair by taking the average element-wisely to obtain the final embedding $z\in\mathbb{R}^m$. 

Considering that the models of the environment are deterministic, we also implement the decoder $q_\psi$ as a deterministic mapping. However, when $q_\psi$ is deterministic, the probability for the predicted label used in Equation~\eqref{sec4:eq_tae} is undefined. To tractably estimate the log probabilities, we follow the common approach of using L2 distances as an approximation for the log probability~\cite{mse1, mse2, mse3}. The original objective $\mathcal{J}(\theta, \psi)$ can thus be equivalently transformed into:
\begin{equation}\label{sec4:eq_tae_impl}
\begin{aligned}
    &\hat{\mathcal{J}}(\theta, \psi) = \\ 
    &\mathbb{E}_{x_{1:n}, M}\Big[\sum_{k=1}^n \Big(y_k^M - q_\psi(q_\theta(x_{1:n}, y_{1:n}), x_k)\Big)^2\Big],
\end{aligned}
\end{equation}    
where $x_{1:n}\sim \rho(x)$, ${M}\sim P(\mathcal{M})$. Finally, we use $\hat{\mathcal{J}}(\theta, \psi)$ as the training objective of the TAE.

\begin{algorithm}[t]
\caption{Data Augmentation via Relabeling}
\label{sec4:code_relabel}
\textbf{Input}: Offline datasets $\{D_{i}\}_{i=1}^N$, pre-trained models $\{\widehat{M}_{i}\}_{i=1}^N$, meta-policy $\pi_\phi$, task representations $\{z_i\}_{i=1}^N$
\begin{algorithmic}[1]
\STATE Initialize augmentation buffers $\{D_{i}^{\text{aug}}\}_{i=1}^N$ as $\emptyset$
\FOR{$i=1, 2, \ldots, N$}
    \FOR{$k=1, 2, ..., K_1$}
        \STATE Sample state $s_k$ from $D_{i}$
        \STATE Sample $a_k\sim \pi_\phi(\cdot |s_k,z_i)$ and $(s'_k, r_k)=\widehat{M}_i(s_k, a_k)$
        \STATE $D_{i}^{\text{aug}} \leftarrow D_{i}^{\text{aug}}\cup \{(s_k, a_k, s'_k, r_k)\}$
    \ENDFOR
    \FOR{$k=1, 2, ..., K_2$}
        \STATE Sample state $s_k$ from $\{D_j\}_{j=1}^N \setminus D_i$
        \STATE Sample $a_k\sim \pi_\phi(\cdot |s_k,z_i)$ and $(s'_k, r_k)=\widehat{M}_i(s_k, a_k)$
        \STATE $D_{i}^{\text{aug}} \leftarrow D_{i}^{\text{aug}}\cup \{(s_k, a_k, s'_k, r_k)\}$
    \ENDFOR
\ENDFOR
\RETURN $\{D_{i}^{\text{aug}}\}_{i=1}^N$
\end{algorithmic}
\end{algorithm}

\subsection{Constructing the Probing Distribution}\label{sec4.3}
The training of TAE is also affected by the distribution of the probing data, whose distribution $\rho(x)$ is desired to satisfy certain requirements. Thus in this section, we investigate how to construct the probing distribution. 

Generally, we expect $\rho(x)$ to satisfy the following properties:

\textbf{1) Independent of models. }This is required by Assumption~\ref{sec4:assump1}, which states that the probing data should be sampled from an invariant distribution regardless of models $M$. 

\textbf{2) Consistent for training and evaluation. }This property requires that $\rho(x)$ should resemble the distribution which the meta-policy may encounter during evaluation. We can perceive the probing distribution $\rho(x)$ as some \textit{attention} over all possible aspects of the models, and the extracted representation $z$ is the most discriminative over $\rho(x)$. 

Given the above two desiderata, we propose to construct the training data of TAE by policy-relabeling, dynamics-relabeling, and reward-relabeling jointly. Suppose the offline datasets are represented by $\{D_{i}\}_{i=1}^N$, we first pretrain an estimated model $\widehat{M}_i$ for each task $M_i$ via supervised learning. At the beginning of each iteration, for each task $M_i$, we randomly pick $K_1$ states from $D_i$ and $K_2$ states from other datasets $\cup_j D_j \setminus D_i$ respectively. For each state $\hat{s}$, we first label its action by sampling $\hat{a}$ from the update-to-date meta-policy parameterized by $\phi$: $\hat{a}\sim \pi_{\phi}(a|s,z_i)$. The state transition $\hat{s}'$ and the reward $\hat{r}$ are further predicted by the pre-trained model $\widehat{M}_i$. The constructed tuple $\langle \hat{s}, \hat{a}, \hat{s}', \hat{r}\rangle$ is thus an augmentation sample used to train TAE. The pseudo-code for the augmentation process is listed in Algorithm~\ref{sec4:code_relabel}. 

By randomly sampling states across all of the datasets and relabeling the actions with the same meta-policy, we align the probing distribution $\rho(x)$ for each task so that the training procedure approximately satisfies property 1). Note that in the actual implementations, we are sampling from the ego dataset and other datasets with a ratio of $K_1:K_2$. Theoretically, the ratio should be set to $1:N-1$ precisely. However, in the practical implementation, we prefer a little biased ratio towards the ego dataset. This is because the estimated model $\widehat{M}_i$ may produce erroneous predictions on the states from other datasets, so we choose to strike a balance with such a biased ratio. For property 2), this is ensured by relabeling actions with the meta-policy. More details about the experiments can be found in the appendix. 

\begin{algorithm}[t]
\caption{Meta Training of GENTLE}
\label{sec4:code_gentle_train}
\textbf{Input}: Offline datasets $\{D_{i}\}_{i=1}^N$, models $\{\widehat{M}_{i}\}_{i=1}^N$, encoder-decoder $q_\theta,q_\psi$,
meta-policy $\pi_\phi$, Q-function $Q_\omega$
\begin{algorithmic}[1]
\STATE Pre-train models $\{\widehat{M}_{i}\}_{i=1}^N$ on $\{D_{i}\}_{i=1}^N$ independently by supervised learning
\FOR{$\text{epoch} = 1,2,\ldots, E$}
    \STATE Augment context data via relabeling: $\{D_{i}^{\text{aug}}\}_{i=1}^N =$ Algorithm~\ref{sec4:code_relabel}
    \FOR{$\text{gradient step} = 1,2,\cdots,S$}
    \STATE // TAE training
    \STATE Sample context batches $\{B_{i}^{\text{c}}\}_{i=1}^N$ from $\{D_{i}^{\text{aug}}\cup D_{i}\}_{i=1}^N$
    \STATE Update $\theta,\psi$ by maximizing Eq.~\ref{sec4:eq_tae_impl} on $\{B_{i}^{\text{c}}\}_{i=1}^N$
    \STATE // Policy optimization
    \STATE Compute representation for each task $\{z_i=q_\theta(B_i^c)\}_{i=1}^N$
    \STATE Sample RL batches $\{B_{i}\}_{i=1}^N$ from offline datasets and concatenate the states in $\{B_{i}\}_{i=1}^N$ with $\{z_i\}_{i=1}^N$
    \STATE Update $\phi,\omega$ on $\{B_{i}\}_{i=1}^N$ by TD3+BC
    \ENDFOR
\ENDFOR
\RETURN $\pi_\phi$
\end{algorithmic}
\end{algorithm}

\begin{table*}[h]
\centering

\begin{subtable}[]{\textwidth}
\centering
\scalebox{0.85}{
\begin{tabular}{c|c|rrrr}
\toprule[1.5pt]
\multicolumn{1}{c}{\textbf{Environment}} & \multicolumn{1}{c}{\textbf{Task Set}}   & \multicolumn{1}{c}{\textbf{FOCAL}} & \multicolumn{1}{c}{\textbf{CORRO}} & \multicolumn{1}{c}{\textbf{BOReL}} & \multicolumn{1}{c}{\textbf{GENTLE} (Ours)}\\
\midrule
Point-Robot & \multirow{6}{*}{Train}   &$-10.04\pm\;\,\,\,\,3.68$    & $\mathbf{-5.76}\pm\;\,\,\,\,1.02$ & $-15.38\pm\;\,3.37$  &  $-6.46\pm\;\,1.57$ \\
Ant-Dir &   & $490.21\pm\;\,73.80$  & $-2.48\pm\;\,14.30$  & $70.32\pm48.77$ &  $\mathbf{570.20}\pm60.81$ \\
Cheetah-Vel &    & $-221.97\pm\;\,44.04$    &   $-384.03\pm\;\,28.34$   &  $-257.82\pm23.87$    &  $\mathbf{-210.77}\pm42.32$      \\
Cheetah-Dir &    & $1449.08\pm182.48$     &   $1350.09\pm124.88$   &   $770.86\pm25.98$       & $\mathbf{1559.95}\pm33.89$  \\
Hopper-Params &    & $343.82\pm\;\,27.83$    &   $154.62\pm\;\,23.86$     & $185.33\pm36.91$   &   $\mathbf{354.58}\pm35.77$    \\
Walker-Params &   & $578.42\pm\;\,53.58$    &    $295.91\pm\;\,34.11$     & $213.21\pm45.96$     & $\mathbf{627.77}\pm41.10$    \\
\midrule
Point-Robot & \multirow{5}{*}{Test}    & $-13.94\pm\;\,\,\,\,2.66$   & $-11.38\pm\;\,\,\,\,0.67$  & $-16.33\pm\;\,2.74$   & $\mathbf{-9.71}\pm\;\,1.31$ \\
Ant-Dir &    & $451.63\pm\;\,80.82$ &  $20.70\pm\;\,17.28$    & $87.57\pm38.90$    & $\mathbf{501.67}\pm98.49$   \\
Cheetah-Vel &  & $\mathbf{-342.14}\pm\;\,66.90$  &    $-552.49\pm\;\,34.44$     &  $-451.63\pm29.05$    & -$362.83\pm46.08$      \\
Hopper-Params &   & $224.56\pm\;\,48.78$   &   $125.75\pm\;\,39.81$     & $149.74\pm14.44$    &  $\mathbf{251.60}\pm14.81$     \\
Walker-Params &    & $277.71\pm\;\,40.32$  &    $245.75\pm\;\,19.05$     &   $175.24\pm25.14$    &  $\mathbf{335.59}\pm51.55$   \\
\bottomrule[1.5pt]
\end{tabular}
}
\end{subtable}
\\[3pt]

\begin{subtable}[]{\textwidth}
\centering
\scalebox{0.85}{
\begin{tabular}{c|c|rrrr}
\toprule[1.5pt]
\multicolumn{1}{c}{\textbf{Environment}} & \multicolumn{1}{c}{\textbf{Task Set}}   & \multicolumn{1}{c}{\textbf{FOCAL}} & \multicolumn{1}{c}{\textbf{CORRO}} & \multicolumn{1}{c}{\textbf{BOReL}} & \multicolumn{1}{c}{\textbf{GENTLE} (Ours)}\\

\midrule

Point-Robot & \multirow{6}{*}{Train}    & $-17.44\pm\;\,\,\,\,3.77$    & $-15.12\pm\;\,\,\,\,1.68$  & $-21.64\pm\;\,6.23$   & $\mathbf{-13.50}\pm\;\;3.26$  \\
Ant-Dir &   &  $188.25\pm\;\,54.58$    & $5.68\pm\;\,30.59$  &$96.71\pm17.34$   & $\mathbf{596.06}\pm78.20$ \\
Cheetah-Vel &    &   $-301.06\pm\;\,32.43$       &  $-450.85\pm\;\,23.98$    &  $\mathbf{-278.21}\pm27.74$     & $-278.95\pm71.80$ \\
Cheetah-Dir &        &  $75.28\pm108.24$     &  $1099.45\pm546.08$    & $764.59\pm12.80$   &  $\mathbf{1525.66}\pm70.05$ \\
Hopper-Params &      & $192.08\pm\;\,45.43$       &  $129.44\pm\;\,22.15$   & $44.95\pm16.24$     & $\mathbf{234.63}\pm24.74$   \\
Walker-Params &      &  $294.75\pm\;\,34.19$      &  $279.76\pm\;\,52.59$    & $207.87\pm47.08$     & $\mathbf{356.70}\pm40.32$ \\
\midrule
Point-Robot & \multirow{5}{*}{Test}    & $-18.39\pm\;\,\,\,\,3.47$  & $-17.16\pm\;\,\,\,\,1.56$ & $-20.37\pm\;\,2.05$    & $\mathbf{-17.02}\pm\;\;2.60$ \\
Ant-Dir &   &  $103.20\pm\;\,56.39$    & $16.44\pm\;\,21.47$  & $63.08\pm33.67$   & $\mathbf{464.11}\pm95.74$  \\
Cheetah-Vel &    &  $-377.62\pm\;\,86.74$       &  $-616.02\pm\;\,31.93$    &  $-466.94\pm39.26$     & $\mathbf{-369.93}\pm64.77$ \\
Hopper-Params &      & $194.61\pm\;\,65.53$      &  $135.77\pm\;\,21.46$   &  $53.58\pm21.00$   &  $\mathbf{221.35}\pm27.32$  \\
Walker-Params &      &   $244.03\pm\;\,31.35$       &  $217.82\pm\;\,38.90$    &  $174.85\pm20.93$   & $\mathbf{300.66}\pm48.78$ \\
\bottomrule[1.5pt]
\end{tabular}
}
\end{subtable}

\caption{Performance on the benchmarks. Each number represents the return of the last checkpoint of the meta-policy, averaged over 8 random seeds, $\pm$ represents standard deviation. Top: given-context performance. Bottom: one-shot performance.}
\label{sec5.2:baselines}
\end{table*}

\subsection{Overall Framework of GENTLE}\label{sec4.4}
We summarize the overall meta-training framework of GENTLE in Algorithm~\ref{sec4:code_gentle_train}. At the beginning of each iteration, we first augment context data with the update-to-date meta-policy and the pre-trained models to construct the probing data. The probing data is then used to optimize the TAE as well as to compute the task representation. After detaching the gradient w.r.t. the encoder, the representation will be concatenated to raw observations for downstream offline policy optimization, which is implemented as TD3+BC. GENTLE iterates between the construction of probing data and the optimization of the meta policy until convergence. 

At test time, we test GENTLE with both \textit{given-context} protocol and \textit{one-shot} protocol. The former assumes that the meta-policy is given access to a dataset collected in testing task to serve as context $B^c$. In the latter, we first collect a trajectory as context $B^c$ with $z_{\text{prior}}$ sampled from a prior distribution, and calculate task representation $z=q_\theta(B^c)$. Then the meta-policy is evaluated by conditioning on the calculated representation. In the practical implementation, we scale the range of $z$ to $(-1,1)$ and set $z_{\rm prior}$ to all zeros. 

\section{Experiments}
\begin{table*}[htbp]
\centering
\scalebox{0.85}{
\begin{tabular}{c|c|rrrr}
\toprule[1.5pt]
\multicolumn{1}{c|}{\multirow{2}{*}{\textbf{Environment}}} & \multirow{2}{*}{\textbf{Task Set}}  & \multicolumn{1}{c}{\textbf{GENTLE}}  &\multicolumn{1}{c}{\textbf{GENTLE w/o}}  & \multicolumn{1}{c}{\textbf{GENTLE w/o}} & \multicolumn{1}{c}{\multirow{2}{*}{\textbf{GENTLE}}}        \\  
\multicolumn{1}{c|}{}  &   &   \multicolumn{1}{c}{\textbf{Contrastive}}  & \multicolumn{1}{c}{\textbf{Relabel}} &  \multicolumn{1}{c}{\textbf{PolicyRelabel}}  &       \\ 

\midrule

Point-Robot & \multirow{6}{*}{Train}    & $-18.50\pm\;\,\,\,\,3.20$   &$\mathbf{-12.80}\pm\;\,\,\,\,0.98$    & $-16.91\pm\;\,3.59$   & $-13.50\pm\;\,3.26$  \\
Ant-Dir &   &  $509.20\pm106.19$   &$206.56\pm\;\,78.00$  & $437.68\pm91.65$    & $\mathbf{596.06}\pm78.20$ \\
Cheetah-Vel &    & $\mathbf{-263.57}\pm\;\,39.01$        &  $-367.84\pm\;\,51.50$    &   $-304.29\pm60.17$    & $-278.95\pm71.80$ \\
Cheetah-Dir &        &  $1465.21\pm141.88$  &  $91.33\pm726.57$     & $1512.32\pm89.38$  &  $\mathbf{1525.66}\pm70.05$ \\
Hopper-Params &      & $199.94\pm\;\,34.45$   &$164.95\pm\;\,51.91$      &     $200.13\pm35.19$   & $\mathbf{234.63}\pm24.74$   \\
Walker-Params &      &  $284.66\pm\;\,53.11$        & $329.40\pm\;\,46.80$     & \;\;\;$336.03\pm50.99$    & $\mathbf{356.70}\pm40.32$ \\
\midrule
Point-Robot & \multirow{5}{*}{Test}    & $-19.21\pm\;\,\,\,\,1.47$  & $\mathbf{-15.51}\pm\;\,\,\,\,1.36$    & $-17.71\pm\;\,3.32$   & $-17.02\pm\,\,\, 2.60$ \\
Ant-Dir &   &  $376.37\pm110.18$ & $125.99\pm\;\,61.42$   & $356.20\pm75.99$   & $\mathbf{464.11}\pm95.74$  \\
Cheetah-Vel &    &  $-415.44\pm\;\,45.16$   & $-439.77\pm\;\,63.57$     & $-448.27\pm76.54$    & $\mathbf{-369.93}\pm64.77$ \\
Hopper-Params &      & $199.78\pm\;\,31.75$    &  $144.22\pm\;\,56.86$     &  \;\;\;$188.71\pm29.90$   &  $\mathbf{221.35}\pm27.32$  \\
Walker-Params &      &   $232.79\pm\;\,36.21$         & $279.40\pm\;\,66.80$      &   \;\;\,$292.05\pm61.89$   & $\mathbf{300.66}\pm48.78$ \\
\bottomrule[1.5pt]
\end{tabular}
}
\caption{Performance of GENTLE variants on one-shot protocol. Each number represents the return of the last checkpoint of the meta-policy, averaged over 8 random seeds, $\pm$ represents the standard deviation.}
\label{sec5.3:ablation}
\end{table*}
In this section, we carry out extensive experiments to compare GENTLE against other OMRL baseline methods and provide an ablation study on the design choices of GENTLE. We also illustrate the learned task representations as a qualitative assessment of the proposed method. We release our code at \url{https://github.com/LAMDA-RL/GENTLE}.

\subsection{Baselines and Benchmarks}
Following the experimental setup in prior studies~\cite{focal,corro}, we construct a 2D navigation environment and several multi-task MuJoCo~\cite{mujoco} environments to evaluate our algorithm. 
To comply with the data limitation on the number of training tasks, we sample 10 training tasks and 10 testing tasks for each environment (except for Cheetah-Dir which only has two tasks). For offline dataset generation, we train a SAC~\cite{sac} agent to expert level on each task and then collect trajectories as the offline datasets to simulate the data limitation on behavior diversity. 

To evaluate the performance of GENTLE, We compare it with the following OMRL methods: \textbf{FOCAL}~\cite{focal} employs distance metric learning to train the context encoder. \textbf{CORRO}~\cite{corro} employs contrastive learning to train the context encoder using InfoNCE loss. \textbf{BOReL}~\cite{borel} employs a variational autoencoder to learn task embeddings by maximizing of the evidence lower bound.

Note that in the original implementation, FOCAL uses BRAC~\cite{brac}, CORRO and BOReL use SAC as their backbone offline RL algorithms, while we use TD3+BC as the backbone algorithm. To provide a fair comparison, we also implement them with TD3+BC, and conduct the experiments with both the original baselines and the re-implemented baselines. We place the results of the original baselines in the appendix
. Finally, we use a variant of BOReL without oracle reward relabeling in the experiments for a fair comparison.

\subsection{Main Results}
We evaluate GENTLE alongside other baselines across the given-context protocol and the one-shot protocol. As shown in Table \ref{sec5.2:baselines}, it is evident that GENTLE significantly outperforms other baselines in almost all scenarios. 
When the context is given, all considered methods exhibit reasonable performance. However, it is noteworthy that GENTLE slightly surpasses the performance of the baseline methods, which indicates the efficacy of the representations extracted by GENTLE. We witness a sharp drop in the baseline methods such as FOCAL when switching to one-shot protocol, which is primarily attributed to the context distribution shift between training and online adaptation. In contrast, GENTLE remarkably sustains a high-performing policy under the one-shot protocol, thereby showcasing the remarkable generalization capability of GENTLE's representation encoder during online adaptation.

\subsection{Illustration of the Representations}\label{sec5.3}
We dive deeper to examine the learned representations of each algorithm. To illustrate the quality of the learned representations, we use the meta-policy to collect context data in the testing tasks and employ the learned encoder to predict the task representations. For each task, we obtain a total number of 400 representation vectors and employ t-SNE~\cite{tsne} to project them onto a two-dimensional plane. The results, as depicted in Figure \ref{sec5.3:tsne}, reveal that the representations predicted by GENTLE naturally form distinct clusters based on their task IDs, signifying the proficiency of GENTLE's encoder in deriving effective task representation even for the testing tasks. On the contrary, FOCAL, CORRO, and BOReL fail to distinguish the representations with online context. The predicted task representations are intertwined and lack clear distinctions in the projected space.

\begin{figure}[ht]
  \centering
  \includegraphics[width=0.8\linewidth]{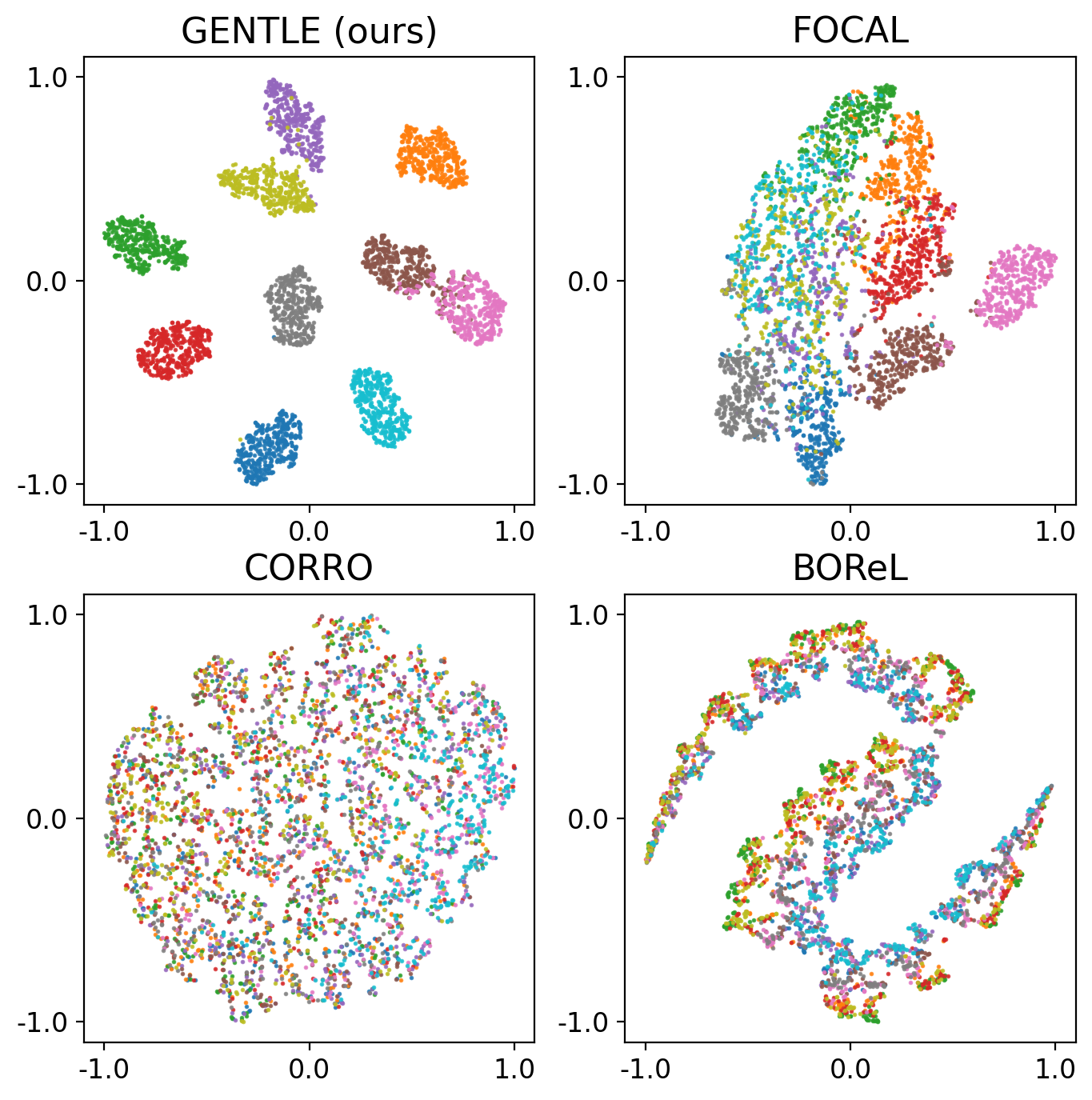}
  \caption{Visualization of the learned representations drawn from 10 testing tasks in Ant-Dir. Each point represents an embedding vector extracted from online context, which is color-coded according to task identity.}
  \label{sec5.3:tsne}
\end{figure}

\subsection{Ablation Study}\label{sec5.4}
\subsubsection{Ablation on algorithm components.}
The core ingredients and innovations of GENTLE are 1) the TAE structure trained by reconstruction of rewards and transitions, and 2) the construction of probing data used to train TAE. We investigate the necessity of these components.

We introduce several variants of GENTLE. Specifically, we replace TAE's objective with the contrastive-style objective used in FOCAL, and term this variant as GENTLE-Contrastive. We create another variant, GENTLE without Relabel, by skipping the relabeling process and training TAE directly with the offline datasets. The last variant, GENTLE without PolicyRelabel, skips policy-relabeling and uses the dataset action for relabeling. The results under one-shot protocol are listed in Table \ref{sec5.3:ablation}. 
By comparing GENTLE and GENTLE-Contrastive, we find that the reconstruction objective does offer benefits over the contrastive-style objective. The variant without relabeling shows severely degenerated performance in certain tasks, particularly in Ant-Dir and Cheetah-Dir. Without the process of relabeling, the context encoder tends to overfit to training data distribution. Finally, although the variant without policy-relabeling shows favorable results on all tasks, its performance still lags behind GENTLE. This exemplifies the importance of the consistency property for the probing data distribution. 

\subsubsection{Ablation on sampling ratio.}
To construct the probing data, we sample them from the ego dataset and other datasets with a ratio of $K_1:K_2$. We investigate the influence of this ratio. We conduct experiments on Ant-Dir, with a series of sampling ratios: 1:0, 1:1, 1:3, 1:6, 1:9, 1:12, 1:15. Specifically, a ratio of 1:0 signifies exclusive sampling from the ego task dataset, while a ratio of 1:9 implies comprehensive sampling from all other task datasets. And for ratios 1:12 and 1:15, we downsample the ego task dataset. The results are depicted in Figure \ref{sec5.4:ratio}.
The ratio of 1:0 exhibits the poorest performance, attributed to its reliance solely on the ego dataset which fails to ensure property 1) in Section \ref{sec4.3}. Increasing the sampled number of other task datasets leads to improved performance. Notably, larger ratios yield similar or slightly worse performance, and ratios 1:12 and 1:15 result in instability and performance drop, which, as we stated before, can be attributed to the estimation error of the pre-trained dynamics models. Besides, a larger ratio also requires more computation. Based on the above considerations, we opt for a balanced ratio of 1:3 in all of the experiments.  
\begin{figure}[ht]
  \centering
  \includegraphics[width=1.0\linewidth]{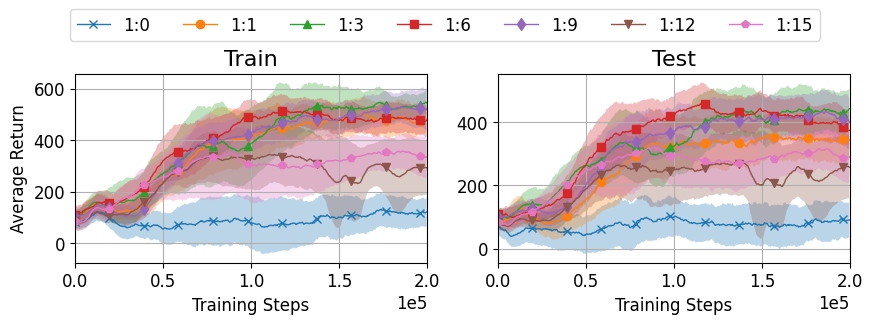}
  \caption{Return curves of GENTLE on training and testing tasks in Ant-Dir across 7 sampling ratios.} 
  \label{sec5.4:ratio}
\end{figure}

\subsubsection{Ablation on training tasks and behavior diversity.}
GENTLE is proposed to tackle data limitations on training tasks and behavior diversity. To inspect how GENTLE adapts to the former, we vary the number of training tasks between 4-10 while leaving testing tasks unchanged. As shown in Figure~\ref{sec5.4:task_num}, a small number of training tasks significantly diminishes GENTLE's generalization over testing tasks. With an increase in the number of training tasks, GENTLE exhibits enhanced generalization performance.
\begin{figure}[ht]
  \centering
  \includegraphics[width=0.9\linewidth]{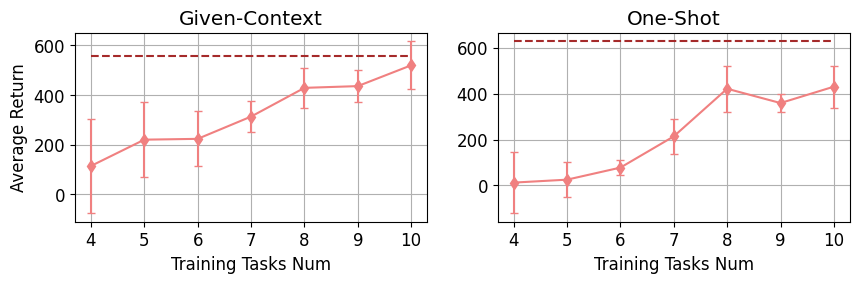}
  \caption{Performance of GENTLE on Ant-Dir testing tasks under different numbers of training tasks across 8 seeds. Dotted line represents the performance on training tasks.}
  \label{sec5.4:task_num}
\end{figure}


Lastly, we illustrate the capability of handling limited behavior diversity by inspecting how algorithms adapt to improved diversity. We use a medium-level policy to collect \textit{medium} context data, and use 5 logged checkpoints of policy to collect \textit{mixed} context data. The \textit{medium} and \textit{mixed} data are only used to train the context encoder, while the policy is still optimized with expert datasets. As shown in Table~\ref{table:diversity}, FOCAL's performance witnessed significant improvement as the behavior diversity improves (from \textit{expert} to \textit{mixed}), while GENTLE remains approximately the same since the data-relabeling process already enriches the behavior diversity and aligns the distribution even with the \textit{expert} context.

\begin{table}[htbp]
\centering
\scalebox{0.7}{
\begin{tabular}{c|c|rrr}
\toprule[1.5pt]
\textbf{Algorithm} & \textbf{Task Set} & \textbf{Expert} &  \textbf{Medium} &  \textbf{Mixed}  \\
\midrule
\multirow{2}{*}{\textbf{FOCAL}} & Train & $188.25\pm54.58$ & $235.69\pm97.81$  & $\mathbf{353.43}\pm68.28$ \\ 
  &Test & $103.20\pm56.39$ & $124.58\pm58.95$ & $\mathbf{213.48}\pm81.40$    \\ 
\midrule
\multirow{2}{*}{\textbf{GENTLE}} & Train &  $\mathbf{596.06}\pm78.20$ & $507.32\pm50.18$ & $571.41\pm97.11$ \\ 
 & Test &  $464.11\pm95.74$ & $475.11\pm85.40$ & $\mathbf{481.10}\pm80.07$  \\ 
\bottomrule[1.5pt]
\end{tabular}
}
\caption{Performance on one-shot protocol in Ant-Dir across 8 seeds under different types of context training datasets.}
\label{table:diversity}
\end{table}

\section{Conclusion and Future Work}
In this paper, we propose an innovative OMRL algorithm called GENTLE. We adopt a novel structure of Task Auto-Encoder (TAE), which incorporates an encoder-decoder framework trained by reconstruction of rewards and transitions. We also employ relabeling to construct pseudo-transitions, which aligns the TAE's training data distribution with the testing data distribution during meta-adaptation. Our experimental results show GENTLE's superior performance in diverse environments and tasks. The ablation studies emphasize the necessity of each component of GENTLE.

Notwithstanding the achievements, our work leaves some aspects unaddressed. We lack provisions for sparse reward settings, and do not tackle the development of exploration policy during meta-testing. We leave these for future work.

\section*{Acknowledgements}
This work is supported by the National Key R\&D Program of China (2022ZD0114804), the National Science Foundation of China (62276126, 62250069), and the Fundamental Research Funds for the Central Universities (14380010).

\bibliography{aaai24}

\begin{thebibliography}{32}
\providecommand{\natexlab}[1]{#1}

\bibitem[{Babaeizadeh et~al.(2018)Babaeizadeh, Finn, Erhan, Campbell, and
  Levine}]{mse2}
Babaeizadeh, M.; Finn, C.; Erhan, D.; Campbell, R.~H.; and Levine, S. 2018.
\newblock Stochastic Variational Video Prediction.
\newblock In \emph{International Conference on Learning Representations
  (ICLR)}.

\bibitem[{Cho, Jung, and Sung(2022)}]{mtsac}
Cho, M.; Jung, W.; and Sung, Y. 2022.
\newblock Multi-Task Reinforcement Learning with Task Representation Method.
\newblock In \emph{ICLR Workshop on Generalizable Policy Learning in Physical
  World}.

\bibitem[{Chung et~al.(2015)Chung, Kastner, Dinh, Goel, Courville, and
  Bengio}]{mse1}
Chung, J.; Kastner, K.; Dinh, L.; Goel, K.; Courville, A.~C.; and Bengio, Y.
  2015.
\newblock A Recurrent Latent Variable Model for Sequential Data.
\newblock In \emph{Advances in Neural Information Processing Systems (NIPS)},
  2980--2988.

\bibitem[{Dorfman, Shenfeld, and Tamar(2021)}]{borel}
Dorfman, R.; Shenfeld, I.; and Tamar, A. 2021.
\newblock Offline Meta Reinforcement Learning - Identifiability Challenges and
  Effective Data Collection Strategies.
\newblock In \emph{Advances in Neural Information Processing Systems
  (NeurIPS)}, 4607--4618.

\bibitem[{Duan et~al.(2016)Duan, Schulman, Chen, Bartlett, Sutskever, and
  Abbeel}]{rl2}
Duan, Y.; Schulman, J.; Chen, X.; Bartlett, P.~L.; Sutskever, I.; and Abbeel,
  P. 2016.
\newblock RL{\^{}}2: Fast Reinforcement Learning via Slow Reinforcement
  Learning.
\newblock \emph{arXiv preprint arXiv:1611.02779}.

\bibitem[{Finn, Abbeel, and Levine(2017)}]{maml}
Finn, C.; Abbeel, P.; and Levine, S. 2017.
\newblock Model-Agnostic Meta-Learning for Fast Adaptation of Deep Networks.
\newblock In \emph{International Conference on Machine Learning (ICML)},
  1126--1135.

\bibitem[{Fu et~al.(2021)Fu, Tang, Hao, Chen, Feng, Li, and Liu}]{towards}
Fu, H.; Tang, H.; Hao, J.; Chen, C.; Feng, X.; Li, D.; and Liu, W. 2021.
\newblock Towards Effective Context for Meta-Reinforcement Learning: An
  Approach Based on Contrastive Learning.
\newblock In \emph{AAAI Conference on Artificial Intelligence (AAAI)},
  7457--7465.

\bibitem[{Fujimoto and Gu(2021)}]{td3bc}
Fujimoto, S.; and Gu, S.~S. 2021.
\newblock A Minimalist Approach to Offline Reinforcement Learning.
\newblock In \emph{Advances in Neural Information Processing Systems
  (NeurIPS)}, 20132--20145.

\bibitem[{Fujimoto, Meger, and Precup(2019)}]{bcq}
Fujimoto, S.; Meger, D.; and Precup, D. 2019.
\newblock Off-Policy Deep Reinforcement Learning without Exploration.
\newblock In \emph{International Conference on Machine Learning (ICML)},
  2052--2062.

\bibitem[{Fujimoto, van Hoof, and Meger(2018)}]{td3}
Fujimoto, S.; van Hoof, H.; and Meger, D. 2018.
\newblock Addressing Function Approximation Error in Actor-Critic Methods.
\newblock In \emph{International Conference on Machine Learning (ICML)},
  1582--1591.

\bibitem[{Haarnoja et~al.(2018)Haarnoja, Zhou, Abbeel, and Levine}]{sac}
Haarnoja, T.; Zhou, A.; Abbeel, P.; and Levine, S. 2018.
\newblock Soft Actor-Critic: Off-Policy Maximum Entropy Deep Reinforcement
  Learning with a Stochastic Actor.
\newblock In \emph{International Conference on Machine Learning (ICML)},
  1856--1865.

\bibitem[{Kirk et~al.(2023)Kirk, Zhang, Grefenstette, and
  Rockt{\"{a}}schel}]{generalization_survey}
Kirk, R.; Zhang, A.; Grefenstette, E.; and Rockt{\"{a}}schel, T. 2023.
\newblock A Survey of Zero-Shot Generalisation in Deep Reinforcement Learning.
\newblock \emph{Journal of Artificial Intelligence Research}, 76: 201--264.

\bibitem[{Kostrikov, Nair, and Levine(2022)}]{iql}
Kostrikov, I.; Nair, A.; and Levine, S. 2022.
\newblock Offline Reinforcement Learning with Implicit Q-Learning.
\newblock In \emph{International Conference on Learning Representations
  (ICLR)}.

\bibitem[{Kumar et~al.(2019)Kumar, Fu, Soh, Tucker, and Levine}]{bear}
Kumar, A.; Fu, J.; Soh, M.; Tucker, G.; and Levine, S. 2019.
\newblock Stabilizing Off-Policy Q-Learning via Bootstrapping Error Reduction.
\newblock In \emph{Advances in Neural Information Processing Systems
  (NeurIPS)}, 11761--11771.

\bibitem[{Kumar et~al.(2020)Kumar, Zhou, Tucker, and Levine}]{cql}
Kumar, A.; Zhou, A.; Tucker, G.; and Levine, S. 2020.
\newblock Conservative Q-Learning for Offline Reinforcement Learning.
\newblock In \emph{Advances in Neural Information Processing Systems
  (NeurIPS)}.

\bibitem[{Li et~al.(2020)Li, Vuong, Liu, Liu, Ciosek, Christensen, and
  Su}]{mbml}
Li, J.; Vuong, Q.; Liu, S.; Liu, M.; Ciosek, K.; Christensen, H.~I.; and Su, H.
  2020.
\newblock Multi-Task Batch Reinforcement Learning with Metric Learning.
\newblock In \emph{Advances in Neural Information Processing Systems
  (NeurIPS)}.

\bibitem[{Li, Yang, and Luo(2021)}]{focal}
Li, L.; Yang, R.; and Luo, D. 2021.
\newblock {FOCAL:} Efficient Fully-Offline Meta-Reinforcement Learning via
  Distance Metric Learning and Behavior Regularization.
\newblock In \emph{International Conference on Learning Representations
  (ICLR)}.

\bibitem[{Lin et~al.(2022)Lin, Wan, Xu, Liang, and Zhang}]{merpo}
Lin, S.; Wan, J.; Xu, T.; Liang, Y.; and Zhang, J. 2022.
\newblock Model-Based Offline Meta-Reinforcement Learning with Regularization.
\newblock In \emph{International Conference on Learning Representations
  (ICLR)}.

\bibitem[{Luo et~al.(2022)Luo, Jiang, Yu, Zhang, and Zhang}]{escp}
Luo, F.; Jiang, S.; Yu, Y.; Zhang, Z.; and Zhang, Y. 2022.
\newblock Adapt to Environment Sudden Changes by Learning a Context Sensitive
  Policy.
\newblock In \emph{AAAI Conference on Artificial Intelligence (AAAI)},
  7637--7646.

\bibitem[{Mendonca et~al.(2020)Mendonca, Geng, Finn, and Levine}]{mier}
Mendonca, R.; Geng, X.; Finn, C.; and Levine, S. 2020.
\newblock Meta-Reinforcement Learning Robust to Distributional Shift via Model
  Identification and Experience Relabeling.
\newblock \emph{arXiv preprint arXiv:2006.07178}.

\bibitem[{Mitchell et~al.(2021)Mitchell, Rafailov, Peng, Levine, and
  Finn}]{macaw}
Mitchell, E.; Rafailov, R.; Peng, X.~B.; Levine, S.; and Finn, C. 2021.
\newblock Offline Meta-Reinforcement Learning with Advantage Weighting.
\newblock In \emph{International Conference on Machine Learning (ICML)},
  7780--7791.

\bibitem[{Ni et~al.(2023)Ni, Hao, Mu, Yuan, Zheng, Wang, and
  Liang}]{metadiffuser}
Ni, F.; Hao, J.; Mu, Y.; Yuan, Y.; Zheng, Y.; Wang, B.; and Liang, Z. 2023.
\newblock MetaDiffuser: Diffusion Model as Conditional Planner for Offline
  Meta-RL.
\newblock In \emph{International Conference on Machine Learning (ICML)},
  26087--26105.

\bibitem[{Pong et~al.(2022)Pong, Nair, Smith, Huang, and Levine}]{smac}
Pong, V.~H.; Nair, A.~V.; Smith, L.~M.; Huang, C.; and Levine, S. 2022.
\newblock Offline Meta-Reinforcement Learning with Online Self-Supervision.
\newblock In \emph{International Conference on Machine Learning (ICML)},
  17811--17829.

\bibitem[{Rakelly et~al.(2019)Rakelly, Zhou, Finn, Levine, and Quillen}]{pearl}
Rakelly, K.; Zhou, A.; Finn, C.; Levine, S.; and Quillen, D. 2019.
\newblock Efficient Off-Policy Meta-Reinforcement Learning via Probabilistic
  Context Variables.
\newblock In \emph{International Conference on Machine Learning (ICML)},
  5331--5340.

\bibitem[{Todorov, Erez, and Tassa(2012)}]{mujoco}
Todorov, E.; Erez, T.; and Tassa, Y. 2012.
\newblock MuJoCo: {A} Physics Engine for Model-Based Control.
\newblock In \emph{{IEEE/RSJ} International Conference on Intelligent Robots
  and Systems (IROS)}, 5026--5033.

\bibitem[{van~den Oord, Li, and Vinyals(2018)}]{cpc}
van~den Oord, A.; Li, Y.; and Vinyals, O. 2018.
\newblock Representation Learning with Contrastive Predictive Coding.
\newblock \emph{arXiv preprint arXiv:1807.03748}.

\bibitem[{Van~der Maaten and Hinton(2008)}]{tsne}
Van~der Maaten, L.; and Hinton, G. 2008.
\newblock Visualizing Data using t-SNE.
\newblock \emph{Journal of Machine Learning Research}, 9(86): 2579--2605.

\bibitem[{Wu, Tucker, and Nachum(2019)}]{brac}
Wu, Y.; Tucker, G.; and Nachum, O. 2019.
\newblock Behavior Regularized Offline Reinforcement Learning.
\newblock \emph{arXiv preprint arXiv:1911.11361}.

\bibitem[{Yang et~al.(2020)Yang, Petersen, Zha, and Faissol}]{sept}
Yang, J.; Petersen, B.~K.; Zha, H.; and Faissol, D.~M. 2020.
\newblock Single Episode Policy Transfer in Reinforcement Learning.
\newblock In \emph{International Conference on Learning Representations
  (ICLR)}.

\bibitem[{Yuan and Lu(2022)}]{corro}
Yuan, H.; and Lu, Z. 2022.
\newblock Robust Task Representations for Offline Meta-Reinforcement Learning
  via Contrastive Learning.
\newblock In \emph{International Conference on Machine Learning (ICML)},
  25747--25759.

\bibitem[{Zhang et~al.(2021)Zhang, Wang, Hu, Chen, Chen, Fan, and Zhang}]{mse3}
Zhang, J.; Wang, J.; Hu, H.; Chen, T.; Chen, Y.; Fan, C.; and Zhang, C. 2021.
\newblock MetaCURE: Meta Reinforcement Learning with Empowerment-Driven
  Exploration.
\newblock In \emph{International Conference on Machine Learning (ICML)},
  12600--12610.

\bibitem[{Zintgraf et~al.(2020)Zintgraf, Shiarlis, Igl, Schulze, Gal, Hofmann,
  and Whiteson}]{varibad}
Zintgraf, L.~M.; Shiarlis, K.; Igl, M.; Schulze, S.; Gal, Y.; Hofmann, K.; and
  Whiteson, S. 2020.
\newblock VariBAD: {A} Very Good Method for Bayes-Adaptive Deep {RL} via
  Meta-Learning.
\newblock In \emph{International Conference on Learning Representations
  (ICLR)}.

\end{thebibliography}

\end{document}